\documentclass[10pt, twocolumn, letterpaper]{article}
\usepackage{graphicx} % for pdf, bitmapped graphics files
\usepackage{bm}
\usepackage{enumitem}
\usepackage{amsmath}
\usepackage{amsthm}
\usepackage{amssymb}
\usepackage{mathtools}
\usepackage{cuted}
\usepackage[textsize=tiny,textwidth=0.6in]{todonotes}
\usepackage{xspace}
\usepackage{mathtools}
\usepackage{definitions}
\usepackage[numbers]{natbib}
\usepackage{authblk}
\title{\LARGE \bf
Improved PAC-Bayesian Bounds for Linear Regression
}
\author[1]{Vera Shalaeva}
\author[2]{Alireza Fakhrizadeh Esfahani}
\author[3]{Pascal Germain}
\author[4]{Mihaly Petreczky}
\affil[1,3]{\small The MODAL project-team, INRIA Lille Nord-Europe, France \protect\\ \{vera.shalaeva, pascal.germain\}@inria.fr}
\affil[2,4]{Univ. Lille, CNRS, Centrale Lille, UMR 9189, \protect\\ CRIStAL - Centre de Recherche en Informatique Signal et Automatique de Lille, France \protect\\ alireza.fakhrizadeh@univ-lille.fr,
mihaly.petreczky@centralelille.fr }

\date{}

\begin{document}

\maketitle
\thispagestyle{empty}
\pagestyle{empty}

%%%%%%%%%%%%%%%%%%%%%%%%%%%%%%%%%%%%%%%%%%%%%%%%%%%%%%%%%%%%%%%%%%%%%%%%%%%%%%%%
\begin{abstract}
In this paper, we improve the PAC-Bayesian error bound for linear regression derived in 
\citet{nips-16}.
%Germain et al. (2016). 
The improvements are two-fold. First, the proposed error bound is tighter, and converges to the generalization loss with a well-chosen temperature parameter. Second, the error bound also holds for training data that are not independently sampled. In particular, the error bound applies to certain time series generated by well-known classes of dynamical models, such as ARX models.
\end{abstract}

\section{Introduction}
 
When facing a machine learning problem, one must be careful to avoid overfitting the training dataset. Indeed, it is well known that minimizing the empirical prediction error is not sufficient to generalize to future observations. 
This is especially important for sensitive ``AI'' applications that are nowadays tackled by many industries (self-driving vehicles, health diagnosis, personality profiling, to name a few).
Statistical learning theories study the generalization properties of learning algorithms. For the prediction problems, they provide guarantees on the ``true'' error of machine learning predictors (\ie, the probability of erroneously predicting the labels of not seen yet samples).

The PAC-Bayesian learning theory, initiated by David McAllester (\citeyear{mcallester-99,mcallester-03b})---see \citet{guedj2019primer} for a recent survey---, has the particularity of providing computable ``non-vacuous'' generalization bounds on popular machine learning algorithms, such as neural networks \cite{Dziugaite2017} and SVMs \cite{ambroladze-06}. Moreover, as its name suggests, PAC-Bayesian framework bridges the \emph{frequentist} Probably Approximately Correct theory and the \emph{Bayesian} inference. This topic is namely discussed in \citet{zhang-06}, \citet{grunwald-2012}, \citet{alquier-15}, \citet{nips-16}, \citet{ShethK17}.

In this paper, we build on a result of \citet{nips-16}, which analyses the Bayesian linear regression from a PAC-Bayesian perspective, leading to generalization bounds for the squared loss. We improve the preceding results in two directions.  
First, our new generalization bound is tighter than the one of \citet{nips-16}, and converges to the generalization loss for proper parameters (see Section~\ref{sec:bounds1}). Second, our result holds for training data that are not independently sampled (see Section~\ref{extension}). 
The latter result is directly applicable to the problem of learning dynamical systems from time series data, in particular, to learning ARX models. ARX models are a popular class of dynamical systems with a rich literature
\cite{LjungBook,HannanBook} due to their relative simplicity and 
modelling power. Note that ARX models can be viewed as a simple yet non-trivial subclass of recurrent neural network regressions. 
For example, just like general recurrent neural networks, ARX models have a memory, \ie, they are able to remember past input data. 

Noteworthy, \citet{alquier2012model} proposed PAC-Bayesian oracle inequalities to perform model selection on different time series (weakly dependent processes and causal Bernoulli shifts). 
Thus, their work is complementary to ours, as it  relies on different assumptions and focuses on other types of error bounds. 

\section{PAC-Bayesian Learning}
Let us consider a supervised learning setting, where a learning algorithm is given a training set $S=\{(x_i,y_i)\}^n_{i=1}$ of size~$n$. Each pair $(x_i,y_i)$ links a \emph{description} $x_i\in\Xcal$ to a \emph{label} $y_i \in \Ycal$. Typically, the description is encoded by a real-valued vector ($\Xcal \subseteq \Rbb^d$), and the label is a scalar ($\Ycal\subseteq\Nbb$ for classification problems, or $\Ycal\subseteq \Rbb$ for regression ones). Given $S$, the learning algorithm returns a prediction function $f: \Xcal\to\Ycal$, also referred to as a \emph{hypothesis}. 
We restrict attention to prediction functions/hypotheses that are measurable. 
The ``quality'' of the predictor $f$ is usually assessed through a measurable loss function
$\ell:\Ycal\times \Ycal\to\Rbb$---such as the zero-one loss $\ell(y,y') = \mathbf{1}_{y\neq y'}$ in classification context, or the squared loss $\ell(y,y') = (y-y')^2$ in regression context, by evaluating the \emph{empirical loss}
\begin{equation*}
    \emploss(f)(S) = \frac1n \sum_{i=1}^n \ell (f(x_i), y_i)\,, \quad \mbox{for any $S$.}
\end{equation*}

\paragraph{PAC Learning.}
When facing a machine learning problem, one wants to use $f$ to predict the label $y\in\Ycal$ from a description $x\in\Xcal$ that does not belong to the training set $S$. 
A good predictor ``generalize to unseen data''. 
This is the object of study of the Probably Approximately Correct (PAC) approach \citep{valiant-84}.

In order to study the statistical behavior of the average loss, we introduce the following statistical framework.
We fix a probability space $(\Omega,\ProbData,\SigmaData)$, where $\SigmaData$ is a $\sigma$-algebra over $\Omega$ and 
$\ProbData$ is a probability measure on $\SigmaData$, see for example \citet{Bilingsley} for the terminology. 
We assume that there exist random variables
$\bX_i:\Omega \rightarrow \Xcal$, $\bY_i:\Omega \rightarrow \Ycal$, $i=1,2,\ldots, $, such that 
the description-label pairs $\{(x_i,y_i)\}_{i=1}^{n}$ are samples from the first $n$ variables $\{(\bX_i,\bY_i)\}_{i=1}^{n}$ 
and there exists  $\omega \in \Omega$ such that $x_i=\bX_i(\omega)$, $y_i=\bY_i(\omega)$. 
Moreover, we assume that $\bX_i,\bY_i$ are identically distributed, \ie{} $\EspData g(\bX_i,\bY_i)$ does not depend on $i$ for any
measurable function $g$. 
\begin{notation}[$\EspData$]
We will use $\EspData$  to denote expected value with respect to the measure $\ProbData$.
\end{notation}
 That is, in the sequel, boldface symbols $\mathbf{P}$ and $\mathbf{E}$ will denote the probability and the corresponding mathematical expectation for the data generating distribution, and we will use boldface to denote random variables on the probability space  $(\Omega,\ProbData,\SigmaData)$  and simple font for their samples. 
 As we will see later on, we will also use a probability measure and the corresponding mathematical
expectation defined on the space of predictors, which will be denoted differently. 

The \emph{generalization loss} of a predictor $f$ is then defined as
\begin{equation*}
    \genloss(f) = \EspData \ell (f(\bX_i), \bY_i)\,,
\end{equation*}
and it expresses the average error for ``unseen data''. It is then of interest to compare this error with the average
empirical error, where the average is taken over all possible samples. To this end, we define the random variable
\begin{equation*}
    \emplossStat(f) = \frac1n \sum_{i=1}^n \ell (f(\bX_i), \bY_i)\,,
\end{equation*}
\ie,  for any sample $S=\{(x_i,y_i)=(\bX_i(\omega),\bY_i(\omega)\}_{i=1}^{n}$ of $\{(\bX_i,\bY_i)\}_{i=1}^n$, 
$\emploss(f)(S)=\emplossStat(f)(\omega)$ is a sample of the random variable $\emplossStat(f)$.
By slight abuse of terminology, we will refer to $\emplossStat(f)$ as the \emph{empirical loss} too. 
PAC theories provide upper bounds of the form
\begin{equation*}
    \ProbData \left( \genloss(f) \leq \emplossStat(f) + \varepsilon \right) \,\geq\, 1-\delta\,,
\end{equation*}
where $\delta\in(0,1]$ acts as a ``confidence'' parameter; the whole challenge of the PAC theories is to derive the mathematical expression of $\varepsilon$. 
Among the various approaches proposed to achieve this goal (reviewed in \citet{shalev-14}), we can mention VC-dimension, sample compression, Rademacher's complexity, algorithmic stability, and the PAC-Bayesian theory. 
In the current work, we stand in the PAC-Bayesian learning framework.

\paragraph{PAC-Bayes.} 
The PAC-Bayesian learning framework \citep{mcallester-99,mcallester-03b} has the particularity of reconciling the PAC learning standpoint with the Bayesian paradigm. To be more precise, let us define a $\sigma$-algebra $\SigmaModel$  on the set of predictors $\Fcal$.\footnote{Note that $\SigmaModel$ is completely different from the $\sigma$-algebra $\mathbf{F}$ of the probability space for which the data generating random variables $\bX_i,\bY_i$ are defined. This is not surprising, as the randomness of the data represents an assumption on the nature of the process which generates the data, while $\SigmaModel$ will be used to define probability distributions, which express our subjective preferences for certain predictors, and which will be adjusted based on the observed data.}
\begin{notation}[$\EspModel_{f \sim \rho}$]
 If $\rho$ is a probability distribution function on $\Fcal$, in the sequel we denote by
 $\EspModel_{f \sim \rho}$ the mathematical expectation with respect to the probability measure which corresponds to $\rho$. 
\end{notation}
In the PAC-Bayesian paradigm,
we consider a \emph{prior} probability distribution $\prior$ and a \emph{posterior} probability distribution $\post$ over this $\sigma$-algebra. 
The prior must be chosen independently of the training set $S$, and the learning algorithm role is to output the posterior distribution, instead of a single predictor. The PAC-Bayesian bounds take the form\footnote{Contrary to the example we give here, the relation between the expected empirical loss and the term $\varepsilon$ might be non-linear. This is the case of the famous PAC-Bayes theorem of \citet{seeger-02}.} 
\begin{equation*}
    \ProbData \left( \EspModel_{f\sim\post}\genloss(f) \leq \EspModel_{f\sim\post}\emplossStat(f) + \varepsilon \right) \,\geq\, 1-\delta\,.
\end{equation*}
That is, in the PAC-Bayesian setting, the study focuses on the $\post$-averaged loss.\footnote{The PAC-Bayesian literature also studies the stochastic \emph{Gibbs predictor}, that perform each prediction on $x\in\Xcal$ by drawing $f$ according to $\post$ and outputting $f(x)$ (\eg, \citet{graal-neverending}). } 
Typically, the term $\varepsilon$ takes into account the prior via the Kullback-Leibler divergence:
\begin{equation*}
    \KL(\post\|\prior) = \EspModel_{f\sim\post} \frac{\post(f)}{\prior(f)}\,.
\end{equation*}
Note that KL-diverence is defined only if $\post$ is absolutely continuous with respect to $\prior$.

In this paper, we build on the PAC-Bayesian theorem of  \citet{alquier-15}, which is also the starting point of \citet{nips-16} result improved in upcoming sections.
\begin{thm}[\citet{alquier-15}] \label{thm:general-alquier}
	Given a set $\Fcal$ of measurable hypotheses $\Xcal\to\Ycal$, a measurable
	loss function $\loss:\Ycal\times\Ycal\to\Rbb$,   a prior distribution $\prior$ over $\Fcal$, a $\delta \in (0,1]$, and a real number $\lambda>0$,  
	$\forall \post \text{ over } \Fcal\colon$
	\begin{align} \label{eq:alquier}
        \ProbData\Bigg( &
    \EspModel_{f\sim\post} \genloss(f) 
    	 \le \  \EspModel_{f\sim\post} \emplossStat(f)   \\[-2mm]
    	 &\   + \dfrac{1}{\lambda}\!\left[ \KL(\post\|\prior) +
    	\ln\dfrac{1}{\delta}
    	+ \Psi_{\ell,\prior}(\lambda,n)  \right] \Bigg) \geq 1-\delta \,, \nonumber
    \end{align}
%    \mbox{where }\quad 
    \begin{equation}\label{eq:alquier_assumption}
    \mbox{where}\qquad\Psi_{\ell,\prior}(\lambda,n) = 
    \ln	\EspModel_{f\sim\prior} \EspData
 	e^{\lambda\left(\genloss(f) - \emplossStat(f)\right)}\,.
	\end{equation}
\end{thm}
For completeness, we provide a proof of Theorem~\ref{thm:general-alquier} in Appendix~\ref{appendix:proof-general-alquier}. This proof highlights that the result is obtained without assuming that the random variables $\bX_i$, $\bY_i$ are mutually independent, 
unlike many ``classical'' PAC-Bayesian theorems. However, the  \iid assumption might be necessary to obtain a computable expression from Theorem~\ref{thm:general-alquier}, because it requires bounding the term $\Psi_{\ell,\prior}(\lambda,n)$ of Eq.~\eqref{eq:alquier_assumption}. Indeed, since  $\Psi_{\ell,\prior}(\lambda,n)$ relies on the unknown joint distribution of $\bX_i,\bY_i$ for $i=1,2,\ldots $ 
its approximation needs assumption on the data.

Interestingly, given a training set $S$, obtaining the optimal posterior $\post^*$ minimizing the bound of Theorem~\ref{thm:general-alquier}, does not require evaluating $\Psi_{\ell,\prior}(\lambda,n)$, as this latter term is independent of both $S$ and $\post$. Indeed, for fixed $S$, $\prior$ and $\lambda$, minimizing the right-hand side of Eq.~\eqref{eq:alquier} amounts to solve
         $\post^* = \argmin_\post \big[  \lambda \EspModel_{f\sim\post} \emploss(f)(S) +   \KL(\post\|\prior) \big],$
which is given by the \emph{Gibbs posterior} \cite{catoni-07,alquier-15,guedj2019primer}; for all $f\in\Fcal$, 
 \begin{align} \label{eq:gibbs_posterior}
 \post^*(f) = \frac1Z \prior(f) \exp\left( -\lambda \emploss(f)(S) \right),
 \end{align}
where $Z$ is a normalization constant. 
We refer to $\lambda$ as a \emph{temperature} parameter, as it controls the emphasis on the empirical loss minimization.
The value of $\lambda$ also directly impacts the value of the generalization bound, and the convergence properties of 
 $\Psi_{\ell,\prior}(\lambda,n)$. In particular, if a non-negative loss is upper bounded by a value $L$ (\ie, $\ell(y,y') {\in} [0,L]$ for all $y,y'{\in} \Ycal$), and $\bX_i,\bY_i$ are \iid, we have, for any $f\in\Fcal$ (we provide the mathematical details in Appendix~\ref{appendix}):
\begin{align} \label{eq:leqL}
    \EspData 
    	\exp\left[{\lambda\left(\genloss(f) - \emplossStat(f)\right)}\right]
\leq
    	 \exp\left[{\frac{\lambda^2 L^2}{8 n} }\right].
\end{align}
Hence, we have  $\Psi_{\ell,\prior}(\lambda,n)
 {\leq} \ln	\EspModel_{f\sim\prior}  e^{\frac{\lambda^2 L^2}{8 n} }
 = \frac{\lambda^2 L^2}{8 n}$, from which the following result is obtained.
\begin{cor}
\label{cor:bounded-loss}
	Given $\Fcal$, $\prior$,  a measurable and bounded loss function $\loss:\Ycal\times\Ycal\to[0,L]$, under \iid observations, for $\delta \in (0,1]$ and $\lambda>0$,  for any
	$\post \text{ over } \Fcal\colon$
	\begin{align*}
        \ProbData\Bigg( &
    \EspModel_{f\sim\post} \genloss(f) 
    	 \le \  \EspModel_{f\sim\post} \emplossStat(f)   \\[-2mm]
    	 &\   + \dfrac{1}{\lambda}\!\left[ \KL(\post\|\prior) +
    	\ln\dfrac{1}{\delta}
    	+ \frac{\lambda^2 L^2}{8 n} \right] \Bigg) \geq 1-\delta \,.
    \end{align*}
\end{cor}
\noindent
Therefore, from Corollary~\ref{cor:bounded-loss}, we obtain with $\lambda=\sqrt n$\,,
	\begin{equation}  \label{eq:alquier-sqrtn}
    \EspModel_{\mathclap{f\sim\post}} \genloss(f) 
 \le \,  \EspModel_{\mathclap{f\sim\post}} \emplossStat(f)  
    	  + \dfrac{1}{\sqrt n}\!\left[ \KL(\post\|\prior) {+}
    	\ln\dfrac{1}{\delta} {+} \frac{L^2}{8}
    	 \right].
    \end{equation}  
In turn, with $\lambda=n$\,,
	\begin{equation} \label{eq:alquier-n}
    \EspModel_{\mathclap{f\sim\post}} \genloss(f) 
 \le \  \EspModel_{\mathclap{f\sim\post}} \emplossStat(f)  
    	  + \dfrac{1}{n}\!\left[ \KL(\post\|\prior) {+}
    	\ln\dfrac{1}{\delta}
    	 \right] + \frac{L^2}{8}\,,
    \end{equation}
with probability at least $1$-$\delta$. 
The generalization bound given by Eq.~\eqref{eq:alquier-sqrtn} has the nice property that its value converges to the generalization loss (\ie, the $\frac{1}{\sqrt{n}}\big[\,\cdot\,\big]$ term tends to $0$ as $n$ grows to infinity).
However, the  result of Eq.~\eqref{eq:alquier-n} does not converge: the bound suffers from an additive term $L^2/8$ even with large $n$.

\paragraph{Relation with Bayesian inference.}
Despite its lack of convergence, PAC-Bayesian theorem result of Eq.~\eqref{eq:alquier-n} is interesting for being closely linked to Bayesian inference.
As discussed in \citet{nips-16} (based on earlier results of \citet{zhang-06} and \citet{grunwald-2012}), maximizing the \emph{Bayesian maximum likelihood} amounts to minimize the PAC-Bayes bound of Theorem~\ref{thm:general-alquier} with $\lambda=n$, provided the Bayesian model parameters (typically denoted $\theta$ in the literature) are carefully reinterpreted as predictors (each $\theta$ is mapped to a regressor $f_\theta$), and the considered loss function $\ell$ is the \emph{negative log likelihood} (roughly\footnote{We omit several details here to concentrate on the general idea. We refer the reader to \citet{nips-16} for the whole picture.}, $\loss_{\mathrm{nll}}\big(y,f_\theta(x)\big) = -\ln p(y | x, \theta)$, where $p(y | x, \theta)$ is a Bayesian likelihood).
That is, in these particular conditions, the posterior promoted by the celebrated \emph{Bayesian rule} (\ie, $p(\theta|X,Y) = \frac{p(\theta) p(Y|X,\theta)}{p(Y|X)}$, where $p(\theta)$ is the prior) aligns with the Gibbs posterior of 
Eq.~\eqref{eq:gibbs_posterior}.

Based on this observation, \citet{nips-16} extends Theorem~\ref{thm:general-alquier} to Bayesian linear regression---for which the loss is unbounded--, as discussed in the next section.

\section{Bounds for Bayesian Linear Regression}
\label{sec:bounds1}

In the Bayesian literature \cite[\ldots]{bishop-2006,murphy-12}, it is common to model a linear regression problem by assuming that 
$\Xcal=\mathbb{R}^d$, $\Ycal=\mathbb{R}$.
The input-output pairs 
$\bX_i,\bY_i$ satisfy
 the following assumptions. 
\begin{assumption}~
\label{assumption1}
\begin{description}%[label=(\alph*)]
    % \begin{enumerate}
	\item[(a)] the inputs $\bX_i$ are such that $\bX_i \sim  \Ncal(\zerobf, \sigx^2\Ib)$, and $\bX_i,\bX_j$ are independent for $i \ne j$. 
	%$ \forall \xb \in \Xcal\,,\ \gamma \geq \sup\|\xb\|$\,,
	\item[(b)] the labels are given by $\bY_i = \wb^* \!\cdot \bX_i + \bEpsilon_i$, where $\bEpsilon_i{\sim}\Ncal(0,\sigma_\bEpsilon^2)$ and $\bEpsilon_i, \bEpsilon_j$ are
              independent for $i \ne j$.
% 	\end{enumerate}
\end{description}
\end{assumption}
Here, we consider that $\sigma_\bEpsilon>0$ is fixed, and we want to estimate the weight vector parameters $\wb^*\in\Rbb^d$. 
Thus, the likelihood function of $\bY_i$ given $\bX_i$, $\wb^* \in \Rbb^d$ 
is given by
\begin{multline*}
    p(\bY_i|\bX_i, \wb) {=} \Ncal(\bY_i|\wb\cdot\bX_i, \sigma_\bEpsilon^2) \\
{=}  (2\pi\sigma_\bEpsilon^2)^{-\frac12} e^{\big(\frac{1}{2\sigma_\bEpsilon^2} (\bY_i{-}\wb{\cdot}\bX_i)^2\big)}.
\end{multline*}
Therefore, the corresponding negative log-likelihood loss function is proportional to the \emph{squared loss} of a linear regressor $f_\wb(\xb) = \wb\cdot\xb$\,:
\begin{equation} \label{eq:squared_loss}
\losssqr(f_\wb(\bX_i),\bY_i) \, =\, (\bY_i - \wb\cdot\bX_i)^2\,.
\end{equation}

\newcommand{\mzx}{\mu_{z|\xb}}
\newcommand{\szx}{\sigma_{z|\xb}}
\newcommand{\mbar}{\bar\mu}
\newcommand{\sbar}{\bar\sigma}
\newcommand{\Fcalw}{{\mathcal{F}_{\!d}}}

\subsection{Previous theorem}
Considering a family of linear predictors,  $\Fcalw {=} \{f_\wb|\wb{\in}\Rbb^d\}$, \citet{nips-16} proposed a generalization bound for Bayesian linear regression under the following assumptions. 
% \begin{assumption}
% \label{assumption1.2}
To get a generalization bound for a squared loss in form of Eq.~\eqref{eq:alquier}, one needs to compute the term $\Psi_{\losssqr,\prior}(\lambda, n)$ or upper bound it. The following is the initial PAC-Bayesian bound for unbounded squared loss proposed by \citet{nips-16}. 
\begin{thm}[\citet{nips-16}] \label{thm:pac-bound-squared-prev}
    Given $ \Fcalw, \losssqr$, and $\delta$ defined above, given a  prior distribution $\pi$ over $\Fcalw$ which is a zero mean Gaussian with  covariance $\sigma_\prior^2\Ib$, i.e., $\pi(f_\wb)=\Ncal(\wb|\zerobf, \sigma_\prior^2\,\Ib)$,   under Assumption \ref{assumption1}, for constants $c {\geq} 2\sigx^2\sigma_\prior^2$, and 
    %$\lambda \in (0, \frac{n}{c})$, 
    $\lambda \in (0, \frac{1}{c})$,
    for any posterior distribution $\post$ over $\Fcalw$:
    \begin{align}\label{eq:pac-bound-squared-prev} \nonumber
    & \ProbData \left( ~ ~ \EspModel_{\mathclap{f_\wb\sim\post}} \genloss(f_\wb) 
    	 \le \  \EspModel_{\mathclap{f_\wb\sim\post}} \emplossStat(f_\wb) + \dfrac{1}{\lambda}\!\left[ \KL(\post\|\prior) +
    	\ln\dfrac{1}{\delta}\right]  \right.
    	\\
    &	\left.   + 
    	\frac{\frac12(d + \| \wb^*\|^2) c + (1 - \lambda c)\sigma_\bEpsilon^2}{1-\lambda c} 
    	\right) \geq 1-\delta \,.
    \end{align}
\end{thm}
Theorem~\ref{thm:pac-bound-squared-prev} expresses the result with $\lambda$ stated explicitly, while \citet{nips-16}---see Appendix A.4 therein---were focusing on the case $\lambda=n$. Here, we observe that the bound does not converge; 
regardless the choice of $\lambda$, the last term of Eq.~\eqref{eq:pac-bound-squared-prev} is not negligible.

Note that PAC-Bayesian guarantees for similar Bayesian models has also been proposed by other authors, under different set of assumptions, either bounded loss \cite{ShethK17} or non-random inputs \cite{DalalyanT08}.

\subsection{Improved theorem}
The first contribution of this paper is an improvement of Theorem~\ref{thm:pac-bound-squared-prev}.
\begin{thm} \label{thm:pac-bound-squared}
    Given $ \Fcalw, \losssqr$ defined above,  under Assumption~\ref{assumption1}, 
    for any $\delta \in (0,1]$, $\lambda > 0$,
   for any prior distribution  $\prior$ over $\Fcalw$,  and  for any posterior distribution  $\post$ over $\Fcalw$, the following holds:
   \begin{equation}
         \label{eq:pac-bound-squared}
    \begin{split}
   &  \ProbData\left(~ ~ \EspModel_{\mathclap{f_\wb\sim\post}} \genloss(f_\wb) 
    	\le \  \EspModel_{f_\wb\sim\post} \emplossStat(f_\wb) \right. \\ 
      & \left.
    	  + \dfrac{1}{\lambda}\!\left[ \KL(\post\|\prior) +
    	\ln\dfrac{1}{\delta}
    	+ \Psi_{\losssqr,\prior}(\lambda,n)  \right]\right) \geq 1-\delta\,,
    \end{split}
   \end{equation}
    \begin{align}\label{eq:complexity_term_1}
    \mbox{where}\qquad    \Psi_{\losssqr,\prior}(\lambda, n) 
        &= \ \ln	\EspModel_{f_\wb\sim\prior} \frac{\exp{\left(\lambda v_{\wb} \right) }}{ \left (1 + \frac{\lambda v_{\wb}}{\frac{n}{2}} \right)^{\frac{n}{2}}} \\
        & \leq \ln \EspModel_{f_\wb\sim\prior} \exp \left ( \frac{{\lambda^2 v_{\wb}^2}}{{\frac{n}{2}}} \right) , \label{eq:complexity_term_2}
    \end{align}~~~~~and \qquad $v_{\wb} = \sigx^2 \parallel \wb^* - \wb \parallel_2^2 + \sigma_{\bEpsilon}^2$\,.
\end{thm} 
\begin{proof}
We get the complexity term in form of Eq.~\eqref{eq:complexity_term_1} by simplifying the general form given in Eq.~\eqref{eq:alquier_assumption}, and using assumptions on inputs and a prior distribution.
{%\allowdisplaybreaks[4]
\begin{align*} 
    &\Psi_{\losssqr,\prior}(\lambda, n) \\ 
    &=
    \ln	\!\!\EspModel_{f_\wb\sim\prior} \EspData
    \exp\left[{\lambda \left(\genloss(f_\wb) - \emplossStat(f_\wb)\right)}\right]\\
    &=
    \ln	\!\!\EspModel_{f_\wb\sim\prior} 
    \exp\left({\lambda \genloss(f_\wb)} \right) \EspData \exp{\left(-\lambda \emplossStat(f_\wb)\right)} \\
    &=
    \ln	\!\!\EspModel_{f_\wb\sim\prior} \exp{\Big(\lambda \genloss(f_\wb)\Big)} 
    %\\   &\phantom{{}={}}
    \underbrace{\EspData \exp{\!\Big({-}  \tfrac{\lambda}{n} \sum_{i=1}^{n} (\bY_i {-} \wb {\cdot} \bX_i)^2  \Big)}}_{(\clubsuit)}\!.
\end{align*}
}%
%where
Note that random variable $\bY_i - \wb \cdot \bX_i = (\wb^* - \wb)\bX_i + \bEpsilon_i$ has zero expectation
$$ \EspData (\bY_i - \wb \cdot \bX_i) = (\wb^* - \wb)\EspData \bX_i + \EspData \bEpsilon_i  = 0\,,$$
and its second moment, denoted $v_{\wb}$, which by definition equals $\genloss(f_{\wb})$, is 
\begin{align*}
     \genloss(f_{\wb})&=\EspData (\bY_i - \wb \cdot \bX_i)^2\\
    &=
    \EspData \left[(\wb^* - \wb)\bX^T_i\bX_i(\wb^* - \wb)\right]\\
   &\qquad + \EspData \left[ 2(\wb^* - \wb)\bX_i\bEpsilon_i +  \bEpsilon_i^2 \right ] \\
   & =\ \sigx^2 \parallel \wb^* - \wb \parallel_2^2 + \sigma_{\bEpsilon}^2 \,.
\end{align*}
Hence, $\frac{\bY_i - \wb \cdot \bX_i}{\sqrt{v_{\wb}}}\sim\Ncal(0,1)$ is a normalized random variable, and its squared sum follows Chi-squared distribution law. Note that the term $(\clubsuit)$ of the function $\Psi_{\losssqr,\prior}(\lambda, n)$ in the form 
\begin{equation*}\textstyle
  \EspData 
 \exp\left(-  \frac{\lambda v_{\wb}}{n} \sum_{i=1}^{n} \left( \frac{\bY_i - \wb \cdot \bX_i}{\sqrt{v_{\wb}}}\right)^2  \right),   
\end{equation*}
corresponds to the moment generating function (MGF) of a Chi-squared distribution, \ie{} $(1-2t)^{\frac{n}{2}}$ with $t = - \frac{\lambda v_{\wb}}{n}$.

By replacing the term $(\clubsuit)$ by Chi-Squared MGF and $\genloss(f_{\wb})$ by $v_{\wb}$, we get the complexity term in form of Eq.~\eqref{eq:complexity_term_1}.

Eq.~\eqref{eq:complexity_term_2} is obtained by lower bounding the denominator of  Eq.~\eqref{eq:complexity_term_1}
by using the inequality 
$ (1 + \frac{a}{b})^b > \exp(\frac{ab}{a+b} )$, for $a,b > 0$\,:
\begin{align*}
& \Psi_{\loss, \prior}(\lambda, n) 
= \ \ln	\EspModel_{f_\wb\sim\prior} \frac{\exp{\left(\lambda v_{\wb} \right) }}{ \exp \left( \frac{\lambda v_{\wb} \frac{n}{2}}{\lambda v_{\wb} + \frac{n}{2}} \right) } \\
&= \, \ln \!\!\EspModel_{f_\wb\sim\prior} \!\!\exp \left ( \tfrac{{\lambda^2 v_{\wb}^2}}{{\lambda v_{\wb} + \frac{n}{2}}} \right) 
\, \leq\, \ln \!\!\EspModel_{f_\wb\sim\prior}\!\! \exp \left ( \tfrac{{\lambda^2 v_{\wb}^2}}{{\frac{n}{2}}} \right) . \hspace{2mm} \qedhere
\end{align*} 
\end{proof}

We are interested in the convergence properties of the right side of Eq.~\eqref{eq:pac-bound-squared}. This will highly depend on the choice of $\lambda$.

\begin{itemize}
    \item If $\lambda$ is fixed and does not depend on $n$, and the latter approaches to $\infty$, we get
    \begin{equation*}
        \EspModel_{f_\wb\sim\post} \genloss(f_\wb) 
    	 \le \  \EspModel_{f_\wb\sim\post} \emplossStat(f_\wb)  + \dfrac{1}{\lambda}\!\left[ \KL(\post\|\prior) +
    	\ln\dfrac{1}{\delta} \right].
    \end{equation*}
    The term $\Psi_{\loss,\prior}(\lambda, n)$ amounts to $0$, since the expression under the expectation of Eq.~\eqref{eq:complexity_term_1} will converge to $1$ due to the fact that \[ \exp{(\lambda v_{\wb})} = \lim_{n \rightarrow \infty} \left(1 + \tfrac{\lambda v_{\wb}}{\frac{n}{2}} \right)^{\frac{n}{2}}.\] Hence, an empirical error converges to the generalization error with sufficiently large value of the parameter $\lambda$, and small divergence between prior and posterior distributions.
    \item If $\lambda$ is considered as a function of $n$, then we can obtain convergence of the right side of the Eq.~\eqref{eq:pac-bound-squared} to the left side with a well-chosen temperature parameter. Let $\lambda$ be $n^{\frac{1}{d}}\ln(\frac{1}{\delta})$, then from Eq.~\eqref{eq:pac-bound-squared} and~\eqref{eq:complexity_term_2}, we have
    \begin{align*}
        &\EspModel_{f_\wb\sim\post} \genloss(f_\wb) 
    	 \le \  \EspModel_{f_\wb\sim\post} \emplossStat(f_\wb)  + \frac{\KL(\post\|\prior)}{n^{\frac{1}{d}}\ln(\frac{1}{\delta})} \\
    	 &+ \ 
    	\dfrac{1}{n^{\frac{1}{d}}} + \dfrac{1}{n^{\frac{1}{d}}}\ln \EspModel_{f_\wb\sim\prior} \exp \left ( \frac{{2n^{\frac{2}{d}}\ln(\frac{1}{\delta})^2 v_{\wb}^2}}{n} \right) .
    \end{align*}
    If the amount of training examples $n {\to} \infty $, then the bound converges to generalization loss.
\end{itemize}

\subsection{Theorems comparison}
The new bound given by Theorem~\ref{thm:pac-bound-squared} is always tighter than the previous one of Theorem~\ref{thm:pac-bound-squared-prev}.
Indeed, the fraction of Eq.~\eqref{eq:complexity_term_1} is upper bounded by its numerator $\exp{\left(\lambda v_{\wb} \right) }$. The latter is the exact same expression as in the derivation of \citet{nips-16} (Supp. Material A4, p.11, line 4), which lead us to the prior bound shown in Eq.~\eqref{eq:pac-bound-squared-prev}. Moreover, the new bound converges to zero for well-chosen temperature parameter $\lambda$ as the number of training observations goes to infinity. For these reasons, the result of Theorem~\ref{thm:pac-bound-squared} is strictly stronger than those of Theorem~\ref{thm:pac-bound-squared-prev}.

\section{Extension to the non \iid case}
\label{extension}
 
 In this section we will study the case when the observed data are no longer sampled independently from the underlying distribution.

 \subsection{The learning problem and its relationship with time series}
  We consider the same learning problem as in Section \ref{sec:bounds1},
  but 
  we modify Assumption \ref{assumption1} by no longer assuming that $\bX_i$ are
  \iid random variables, more precisely, we assume the following:
  \begin{assumption}
  \label{assumption2}
   We assume Part \textrm{(b)} of Assumption \ref{assumption1} and we assume
   that $\bX_i \sim  \Ncal(\zerobf, Q_x)$ for some positive definite matrix $Q_x > 0$. 
  \end{assumption}
 It then follows that $\bY_i$ are also identically distributed, 
 $\bY_i \sim \Ncal(0,\sigma_y^2)$, where
 \begin{equation*}
 %\label{model_y}
  \sigma_y^2=\wb^{*T}Q_x \wb^*+\sigma_{\bEpsilon}^2I\,.
 \end{equation*}

Note that from the assumption that $\bX_i$ are identically distributed
  it follows that $\genloss(f_{\wb})$ does not depend on $i$ and 
  \begin{equation*}
   \genloss(f_{\wb})=(\wb^*-\wb)^TQ_x(\wb^* - \wb)+\sigma_{\bEpsilon}^2\,. 
  \end{equation*}

   A particular instance of the learning problem above is the problem of learning ARX models, which
   is a well-studied problem in control theory  and econometrics \cite{LjungBook,HannanBook}. For the sake of
   simplicity, we will deal only with the scalar input, scalar output case. 
   Consider  stationary zero mean discrete-time stochastic processes $\yb_t,\ub_t$,  $t \in \mathbb{Z}$, $t > 0$.

   Assume that
   there exist real numbers $\{a_i,b_i\}_{i=1}^{k}$  and a stochastic process $\eb_t$ such that 
   \begin{equation}
   \label{arx:model}
    \yb_t=\sum_{i=1}^{k} a_i \yb_{t-i} + \sum_{i=1}^{k} b_i \ub_{t-i} + \eb_t\,,
   \end{equation}
   where $\eb_t$ is assumed to be an \iid sequence of random variables such that $\eb_t \in \Ncal(0,\sigma^2)$ and $\eb_t$ is uncorrelated with $\yb_s, \ub_s$ for $s <  t$. 
   Consider the polynomial $\mathbf{a}(z)=z^k-\sum_{i=1}^{k} a_i z^{k-i-1}$. If $\mathbf{a}(z)$ has all its complex roots inside the unit disc, 
   and $\ub_t$ is a stationary, then it is well known \cite{HannanBook} that  there $\yb_t$ is the unique stationary process which satisfies Eq.~\eqref{arx:model}. 
  
   Moreover, if $\ub_t$ is a jointly Gaussian process, then the $\yb_t$ and the parameters $(\{a_i,b_i\}_{i=1}^{k},\sigma^2)$ together with the joint distribution of $\ub_t$ determine the distribution of $\yb_t$ uniquely \cite{HannanBook}.
    
    Intuitively, the learning problem is to try to compute a prediction $\hat{\yb}_t$ of $\yb_t$ based on past values $\{\yb_{t-l},\ub_{t-l}\}_{l=1}^{\infty}$ of the input and output processes. 
  In the literature \cite{HannanBook,LjungBook} one typically would like to minimize the prediction error $\EspData[(\yb_t-\hat{\yb}_t)^2]$
    In principle, this generalization error may depend on $t$. However, if we assume that the predictor $f$ uses only the last $L$ observations and it is of the form
\(
     \hat{\yb}_t=\sum_{i=1}^{L} \hat{a}_i \yb_{t-i} + \sum_{i=1}^{L} \hat{b}_i \ub_{t-i} \),
    %%\end{equation}
    then by stationarity of $\yb_t,\ub_t$, $t \in \mathbb{Z}$, the predictor will not depend on $t$. Furthermore, 
    if $\yb_t,\ub_t$ come from an ARX model Eq.~\eqref{arx:model} and they are Gaussian, then it can be shown \cite{HannanBook} under some mild assumptions
    that the best possible predictor is necessarily of
    the above form  with $L=k$, and in fact, we should take $\hat{a}_i=a_i$, $\hat{b}_i=b_i$, $i=1,\ldots,k$, and in this case the
    generalization error $\EspData[(\yb_t{-}\hat{\yb}_t)^2]=\sigma^2$. 
    For this reason, in the literature \cite{LjungBook,HannanBook} the learning problem is often formulated as the problem of estimating the parameters of the true model
    (Eq.~\eqref{arx:model}). It is well known that for ARX models, the latter point of view is essentially equivalent to finding the predictor for which the
    generalization error $\EspData[(\yb_t-\hat{\yb}_t)^2]$ is the smallest. 

This allows us to   recast the learning problem into our framework for linear regression as follows.
    For every $i=1,2,\ldots$, define
    \begin{equation*}
      \begin{split}
      &  \bY_i=\yb_{i+k}\,,  \\
      &  \bX_i=\begin{bmatrix} \yb_{i+k-1} & \ldots & \yb_{i-1} & \ub_{i+k-1} & \ldots & \ub_{i-1} \end{bmatrix}^T, \\ 
       &  \wb^*=\begin{bmatrix} a_1 & \ldots & a_k & b_1 & \ldots & b_k \end{bmatrix},  \bEpsilon_t=\eb_{i+k}\,.
       \end{split}
     \end{equation*}
    It then follows that $\bX_i,\bY_i,\bEpsilon_i$ satisfy 
    Assumption \ref{assumption2}.

   \subsection{PAC-Bayesian approach for linear regression with possibly dependent observations}
    In this section we discuss the extension of Theorem \ref{thm:pac-bound-squared} to the case when the observations are not independently sampled. 
 
   Although Theorem \ref{thm:general-alquier} holds even when $(\bX_i,\bY_i)$ are not \iid, the proof of 
   Theorem \ref{thm:pac-bound-squared} relies heavily on the independence of $\bX_i$, $i=1,\ldots,n$. 
   More precisely, let us recall from the proof of Theorem
   \ref{thm:pac-bound-squared} the empirical prediction error variables
   \begin{equation}
   \label{pred:error_var}
      \bZ_{\wb,i} = \bY_i - \wb \cdot \bX_i = (\wb^* - \wb) \cdot \bX_i + \bEpsilon_i \,.
   \end{equation}

   The proof of Theorem \ref{thm:pac-bound-squared} relied on $\bZ_{\wb,i}$, $i=1,\ldots,n$ being independent and identically distributed 
   zero mean Gaussian random variables. In our case, the variables $\bZ_{\wb,i}$ are still zero mean  Gaussian variables which are identically
   distributed, but they no longer independent. Hence, we have to take into account the joint distribution of $\{\bZ_{\wb,i}\}_{i=1}^{n}$,
   which in turn depends on the joint distribution of $\{\bX_i\}_{i=1}^{n}$. 
   
  In order to deal with this phenomenon, we will define the \emph{joint covariance matrix} $Q_{X,n}$ of the random variable
  $\bX_{1:n}=\begin{bmatrix} \bX_1^T, & \ldots, & \bX_n^T \end{bmatrix}$ as follows: 
   \begin{equation*}
    Q_{X,n}=\EspData[\bX_{1:n} \bX_{1:n}^T]\,,
   \end{equation*}
  \ie, the $(i,j)$th $d {\times} d$ block matrix element of $Q_{X,n}$ is $\EspData[\bX_i\bX_j^T]$. 
   We can then formulate the following bound.
   \begin{thm}[] \label{thm:pac-bound-squared2}
   Let $\rho_n$ be the minimal eigenvalue of $Q_{X,n} $ and assume that $\rho_n {>} 0$. Under Assumption \ref{assumption2}, for any
  prior distribution $\prior$ over $\Fcalw$, any $\delta \in (0,1]$, any real number $\lambda>0$, 
and for any posterior distribution $\post \text{ over } \Fcalw$, 
we have
%the following holds: 
    \begin{equation} \label{eq:pac-bound-squared2}
	\begin{split} 
     & \ProbData\left( \EspModel_{f_\wb\sim\post} \genloss(f_\wb) 
    	\le \  \EspModel_{f_\wb\sim\post} \emplossStat(f_\wb)  \right. \\
& \left.    	  + \dfrac{1}{\lambda}\!\left[ \KL(\post\|\prior) +
    	\ln\dfrac{1}{\delta}
    	+ \hat{\Psi}_{\loss,\prior}(\lambda,n)  \right]
  \right) \geq 1-\delta\,,
    \end{split}
    \end{equation}
    where
    \begin{align}\label{eq:complexity_term_12}
        &\hat{\Psi}_{\loss,\prior}(\lambda, n) 
        = \ \ln	\EspModel_{f_\wb\sim\prior} \frac{\exp{\left(\lambda v_{\wb} \right) }}{ \left (1 + \frac{\lambda \rho_{n,\wb}}{\frac{n}{2}} \right)^{\frac{n}{2}}} \\
        & \leq \ln \EspModel_{f_\wb\sim\prior} \exp \left ( \frac{\lambda^2 v_{\wb}\rho_{n,\wb}}{{\frac{n}{2}}}+\lambda(v_{\wb}-\rho_{n,\wb}) \right), \label{eq:complexity_term_22}
    \end{align}
    with
    \( v_{\wb}=(\wb^*-\wb)^TQ_x(\wb^* - \wb)+\sigma_{\bEpsilon}^2 \)\,,\\ and
    \( \rho_{n,\wb}=\rho_n (\wb^*-\wb)^T(\wb^* - \wb)+\sigma_{\bEpsilon}^2 \)\,. 
\end{thm}
 \begin{remark}[Comparison with the \iid case]
  If $\bX_i$, $i=1,2,\ldots,$ are independent and $Q_\bX=\sigx^2I_d$, 
 then $Q_{X,n}$ is diagonal, with the diagonal elements being 
  $\sigx^2$. In this case, $\rho_n=\sigx^2$ and 
  $\rho_{n,\wb}=v_{wb}$ and hence the statement of 
  Theorem \ref{thm:pac-bound-squared2}
   boils down to that of Theorem \ref{thm:pac-bound-squared}.
\end{remark}
Before presenting the proof of Theorem \ref{thm:pac-bound-squared2}
some discussion is in order.

Recall that one of the advantages of the error bound of Theorem \ref{thm:pac-bound-squared} was that it
converged to zero as $n \rightarrow \infty$. The question arises if this is the case for 
the error bound of Theorem \ref{thm:pac-bound-squared2}. In order to answer this question
we need  to investigate the dependence on $n$ of the
smallest eigenvalue $\rho_n$ of the covariance matrix $Q_{X,n}$, since
$\rho_n$ is used in the error bound of Theorem \ref{thm:pac-bound-squared2}.
To this end, note that $Q_{X,n}$ is a positive semi-definite matrix, and hence by the properties of
 minimal eigenvalues of positive semi-definite matrices \cite{golub2013matrix}
    $\rho_{n} r^Tr \le  r^TQ_{X,n}r$.
 From \citet{StoicaBook}(Chapter 5, page 135) it follows that
$\rho_{n} \ge \rho_{n-1}$, \ie, $\rho_n$ is a monotonically increasing sequence. In particular, as $\rho_n \le \rho_{1}$ and
$Q_{X,1}=Q_X$, 
$\rho_{1} \|\wb-\wb^*\|_2^2 \le (\wb-\wb^*)^T Q_x (\wb-\wb^{*})$ and
hence $\rho_{n,\wb} \le v_{\wb}$. 
 This means that the right-hand side of Eq.~\eqref{eq:complexity_term_12}
 is not smaller than the right-hand side of Eq.~\eqref{eq:complexity_term_1},
  and Eq.~\eqref{eq:complexity_term_22} is not smaller than 
   Eq.~\eqref{eq:complexity_term_2}.
   
 That is, the error bounds of Theorem \ref{thm:pac-bound-squared2} are not smaller than those of  Theorem \ref{thm:pac-bound-squared}. 
Moreover, $\rho_n \ge 0$ since  it is an eigenvalue of the positive definite matrix $Q_{X,n}$.  
In particular, $\rho_{*}=\lim_{n \rightarrow \infty} \rho_n=\inf_{n} \rho_n$ exists. 

 Then we get the following corollary of Theorem \ref{thm:pac-bound-squared2},
 by noticing that since $\rho_n \ge \rho_*$, 
 $\frac{\exp{\left(\lambda v_{\wb} \right) }}{ \left (1 + \frac{\lambda \rho_{n,\wb}}{\frac{n}{2}} \right)^{\frac{n}{2}}} \le \frac{\exp{\left(\lambda v_{\wb} \right) }}{ \left (1 + \frac{\lambda \rho_{*,\wb}}{\frac{n}{2}} \right)^{\frac{n}{2}}}.$ 

 \begin{cor}
 \label{thm:pac-bound-squared2:col}
  Assume $\rho_* > 0$.
  For any
  prior $\prior$ over $\Fcalw$, any $\delta \in (0,1]$, and any $\lambda>0$, 
and any $\post \text{ over } \Fcalw$,
Eq.~\eqref{eq:pac-bound-squared2} remains true if we replace
$\hat{\Psi}_{\loss,\prior}$ by $\tilde{\Psi}_{\loss,\prior}$, where
    \begin{align*}
        &\hat{\Psi}_{\loss,\prior}(\lambda,n) \le \tilde{\Psi}_{\loss,\prior}(\lambda, n) 
        = \ \ln	\EspModel_{f_\wb\sim\prior} \frac{\exp{\left(\lambda v_{\wb} \right) }}{ \left (1 + \frac{\lambda \rho_{*,\wb}}{\frac{n}{2}} \right)^{\frac{n}{2}}} \,,
    \end{align*}
    with 
    \( v_{\wb}=(\wb^*-\wb)^TQ_x(\wb^* - \wb)+\sigma_{\bEpsilon}^2 \) and
    \( \rho_{*,\wb}=\rho_* (\wb^*-\wb)^T(\wb^* - \wb)+\sigma_{\bEpsilon}^2 \). 
\end{cor}
 Corollary \ref{thm:pac-bound-squared2:col} gives a PAC-Bayesian bound, asymptotic behavior of which is
 easy to study. Indeed,
 since $1 {+} \frac{\lambda \rho_{*,\wb}}{n/2}$ increases with $n$ and 
 it converges to $\exp(\lambda \rho_{*,\wb})$ as $n \rightarrow \infty$, 
 the error bound $\tilde{\Psi}_{\loss,\prior}(\lambda, n)$
 will decrease with $n$ and 
 \begin{equation}
\label{non-id:asymp}
     \lim_{n \rightarrow \infty} \tilde{\Psi}_{\loss,\prior,}(\lambda, n)
     = \ln	\EspModel_{f_\wb\sim\prior} \exp\left(\lambda (v_{\wb} - \rho_{*,\wb}) \right).
 \end{equation}
 That is, contrary to the \iid case in Theorem \ref{thm:pac-bound-squared2}, 
 PAC-Bayesian error bound  of Corollary \ref{thm:pac-bound-squared2:col}
 decreases with $n$, but it will not converge to $0$, rather, it will
 be bounded from above by the right-hand side of Eq.~\eqref{non-id:asymp}.
 Note that $v_{\wb} - \rho_{*,\wb}=(\wb-\wb^*)^T(Q_x-\rho_*I_d)(\wb-\wb^*)$. 
 The latter is a monotonically increasing function of $Q_x-\rho_{*}I_d$: the smaller this difference is, the close
 the right-hand side of Eq.~\eqref{non-id:asymp} to zero.  The difference $Q_x - \rho_*I_d$  is zero in the \iid case,
 and can be seen as a kind of measure of the degree of dependence of  $\bX_i$, $i=1,2,\ldots,$.

Note that Theorem \ref{thm:pac-bound-squared2} and
 Corollary \ref{thm:pac-bound-squared2:col} are meaningful only for $\rho_n > 0$ and $\rho_* > 0$.

For time series assumption that $\rho_{*} > 0$ is equivalent to 
 $Q_{x,n} > mI_{nd}$ for all $n$ for some $m$. 
This property is mild modification of the well-known property of  
\emph{informativity of the data set $\{\yb_t,\ub_t\}_{t=1}^{\infty}$}  
\cite{LjungBook}. This can be seen by an easy modification of the argument of  \citet{StoicaBook}(Chapter 5, page 122, proof of Property 1).
In turn, informativity of the data set is a standard assumption made
in the literature \cite{LjungBook}, and it is required for learning
ARX models.  Note that under mild assumptions on $\ub_t$, from \citet{LjungBook}[Theorem 2.3] it then follows that 
 the $\emplossStat(f_{\wb}) \rightarrow \genloss(f_{\wb})$ as $n \rightarrow \infty$ with probability one.  That is, even though the law of large numbers does not apply in this case, we still know that the empirical loss converges to the generalization error as $n \rightarrow \infty$. 

\begin{proof}[Proof of Theorem \ref{thm:pac-bound-squared2}]
 The proof follows the same lines as that of Theorem \ref{thm:pac-bound-squared2}.  From Theorem \ref{thm:general-alquier} it follows that
\begin{equation} %\label{eq:pac-bound-squared2-1}
\label{thm:pac-bound-squared2:pf-1}
	\begin{split} 
    & \ProbData\left(\EspModel_{f_\wb\sim\post} \genloss(f_\wb) 
    	 \le \  \EspModel_{f_\wb\sim\post} \emplossStat(f_\wb)  \right. \\
   & 	  \left. + \dfrac{1}{\lambda}\!\left[ \KL(\post\|\prior) +
    	\ln\dfrac{1}{\delta}
    	+ \Psi_{\loss,\prior}(\lambda,n)  \right]\right) \geq 1-\delta\,.
    \end{split}
    \end{equation}
 Consider the random variable $Z_{\wb, i}$ defined in Eq.~\eqref{pred:error_var}. 
Just like in
 the proof of Theorem \ref{thm:pac-bound-squared}, 
 \begin{align}
\label{thm:pac-bound-squared2:pf0}
    &\Psi_{\loss,\prior}(\lambda, n)  = \!
    \ln	\!\!\EspModel_{f_\wb\sim\prior}\! \EspData
    \exp\left[{\lambda \left(\genloss(f_\wb) {-} \emplossStat(f_\wb)\right)}\right] \\
    \nonumber
    & = \ln \EspModel_{f_\wb\sim\prior} \left\{\exp\left(\lambda \genloss(f_{\wb})   \right)  \EspData\exp\left(- \frac{\lambda}{n} \sum_{i=1}^{n} \bZ_{\wb, i}^2  \right) \right\}.
\end{align}
 And, it can be shown that $\bZ_{\wb, i}$ is zero mean Gaussian with variance 
 $\EspData[\bZ_{\wb, i}^2]=v_{\wb}$.  In the proof of Theorem \ref{thm:pac-bound-squared}
 we used the fact that under its assumptions 
  $\{\bZ_{\wb, i}\}_{i=1}^{n}$ were mutually independent and identically distributed  and
  hence $\frac{\lambda v_{\wb}}{n} \sum_{i=1}^{n} \frac{\bZ_{\wb, i}}{v_{\wb}^2}$ had
 $\chi^2$ distribution. In our case, $\bZ_{\wb, i}$ are not independent.
 In order to get around this issue, we define the random variable $\bZ_{\wb,1:n}$ and its covariance matrix $Q_{\wb,n}$\,:
 \[ 
   \begin{split}
      \bZ_{\wb,1:n}&=\begin{bmatrix} \bZ_{\wb, i}, & \ldots, & \bZ_{\wb, n} \end{bmatrix}^T, \\
      Q_{\wb,n}&=\EspData[\bZ_{\wb,1:n} \bZ_{\wb,1:n}^T] \,.
  \end{split}
\]
 It is easy to see that
   $Q_{\wb,n} = D_{\wb}^T Q_{X,n} D_{\wb}+\sigma^2_{\bEpsilon} I_n\,, $
  where 
  \[ D_{\wb}=\mathrm{diag}(\underbrace{(\wb-\wb^{*}) I_d,  \ldots, (\wb-\wb^{*}) I_d}_{\mbox{$n$  times}})\,. \]
  Notice that  $r^T Q_{X,n} r \ge \rho_{n} r^Tr$ for all $r \in \mathbb{R}^d$ by \citet{golub2013matrix}.  Then, 
  for any $z \in \mathbb{R}^n$,  by taking $r=D_{\wb}z$, it follows that 
  \begin{equation}
 \label{thm:pac-bound-squared2:pf2}
     \begin{split}
       z^T Q_{\wb,n}z&=(D_{\wb}z)^T Q_{X,n} (D_{\wb}z)+\sigma^2_{\bEpsilon} z^Tz  \\
      & \geq \ \rho_n (D_{\wb}z)^T(D_{\wb}z) + \sigma^2_{\bEpsilon} z^Tz = \rho_{n,\wb}\,. 
     \end{split}
  \end{equation}
  where we used that $\|D_{\wb}z\|^2_2 = \|\wb-\wb^*\|^2_2 \|z\|^2_2$. 
  Define  
  \[ \bS=Q_{\wb,n}^{-1/2}\bZ_{\wb, 1:n}\,. 
\]
 and let $\bS_i$ be the $i$th entry of $\bS$, \ie,  $\bS=\begin{bmatrix} \bS_1 & \ldots & \bS_n \end{bmatrix}^T$.
Then from Eq.~\eqref{thm:pac-bound-squared2:pf2} it follows that 
\begin{equation*}
  \begin{split}
    \sum_{i=1}^{n} \bZ_{\wb, i}^2 &= \bZ_{\wb, 1:n}^T Q_{\wb,n}^{-1/2} Q_{\wb,n} Q_{\wb,n}^{-1/2} \bZ_{\wb, 1:n} \\
&=  \bS^TQ_{\wb,n}\bS  \ge \bS^T\bS \rho_{n,\wb}=   \Big(\sum_{i=1}^{n} \bS_i^2 \Big) \rho_{n,\wb}  \,.
  \end{split}
\end{equation*}
It then follows that 
\begin{equation}
 \label{thm:pac-bound-squared2:pf4}
  \begin{split}
   \exp\left(-  \frac{\lambda}{n}  \sum_{i=1}^{n} \bZ_{\wb, i}^2\right) \le \exp\left(-\frac{\lambda}{n} \rho_{n,\wb}  \sum_{i=1}^{n} \bS_i^2 \right).
  \end{split} 
\end{equation}
 Notice now that $\bS$  is Gaussian and zero mean, with covariance 
      $\EspData[\bS\bS^T]=Q_{\wb,n}^{-1/2} \EspData[\bZ_{\wb,  1:n}\bZ_{\wb, 1:n}^T]Q_{\wb,n}^{-1/2} = I_n$.
 That is, the random variables $\bS_i$ are normally distributed and $\bS_i,\bS_j$ are independent, and therefore $\sum_{i=1}^{n} \bS_i^2$ has $\chi^2$ distribution. Hence,
\[ 
    \EspData\left[\exp\left(-\frac{\lambda \rho_{n,\wb}}{n}  \sum_{i=1}^{n} \bS_i^2 \right)\right]=\frac{1}{(1+\frac{\lambda\rho_{n,\wb}}{\frac{n}{2}})^{\frac{n}{2}}}\,.
\]
Combining this with Eq.\,\eqref{thm:pac-bound-squared2:pf4} and \eqref{thm:pac-bound-squared2:pf0}, Eq.\,\eqref{thm:pac-bound-squared2:pf-1} implies Eq.\,\eqref{eq:complexity_term_12}. 
By using the inequality
$\big(1 {+} \frac{a}{b} \big)^b {>} e^\frac{ab}{a+b}$
for
$a, b {>} 0$
with $a{=}\lambda\rho_{n,\wb}$ and $b{=}\frac{n}{2}$, 
Eq.\,\eqref{eq:complexity_term_22} follows from Eq.\,\eqref{eq:complexity_term_12}.
\end{proof}

\subsection{Related works}

Note that PAC bounds for learning time series has been explored in the literature by \citeauthor{kuznetsov2017generalization} (\citeyear{kuznetsov2017generalization,kuznetsov2018theory}). 
Their approach is based on covering numbers and Rademacher complexity instead of PAC-Bayes analysis, but in contrast to the current paper, \citeauthor{kuznetsov2017generalization}'s work allows for non-stationary time series. 

\citet{alquier2012model} includes a PAC-Bayesian analysis in their model selection procedure for time series. Among other differences, they provide oracle inequalities type of bounds, whereas our analysis provides generalization bounds relying on the empirical loss.

\section{Conclusion}
We have presented an improved PAC-Bayesian error bound for linear regression 
and extended this error bound to the case of non \iid observations. Thus, the obtained bound  
applies to the learning problem of time series using ARX models, which can be viewed as a simple yet non-trivial subclass of recurrent neural network regressions. 
For this reason, we are hopeful that the results of Section~\ref{extension} could potentially lead to PAC-Bayesian bounds for recurrent neural networks.

\section{Acknowledgement}
This work is funded in part by CNRS project PEPS Blanc INS2I 2019 BayesReaForRNN, in part by CPER Data project, co-financed by European Union, European Regional Development Fund (ERDF), French State and the French Region of Hauts-de-France,  and in part by the French project \mbox{APRIORI} ANR-18-CE23-0015.

\appendix
\section{Mathematical details}
\label{appendix}

\subsection{Proof of Theorem~\ref{thm:general-alquier}}
\label{appendix:proof-general-alquier}
\begin{proof} %[Proof of Theorem~\ref{thm:general-alquier}]

The PAC-Bayesian theorem is based on the following \emph{Donsker-Varadhan's change of measure}. 

For any measurable function $\phi:\Fcal\to\Rbb$, we have
	%\begin{equation*} 
	 $\EspModel_{f\sim\post}  \phi(f) \leq \KL(\post\|\prior) + \ln\left( \EspModel_{f\sim\prior} e^{\phi(f)}\right).$
%	\end{equation*}
	Thus, with
	$\phi(f){=} \lambda\big(\genloss(f){-}\emplossStat(f)\big)$, we obtain
	$\forall\, \post \mbox{ on }\Fcal$\,:
	\begin{align}
	\nonumber
	  \EspModel_{f\sim\post} & \lambda\, \big(\genloss(f)-\emplossStat(f)\big)\\
		&\leq \KL(\post\|\prior) + \ln \bigg(
	 \EspModel_{f\sim\prior}  e^{\lambda\, \big(\genloss(f)-\emplossStat(f)\big)}
		\bigg)\,. \label{eq:newhope}
	\end{align}
	Let's consider the random variable 
	$\xi \,{=} \!\displaystyle\EspModel_{f\sim\prior} \!\!e^{\lambda \big(\genloss(f)-\emplossStat(f)\big)}\,.$
	By the Markov inequality, we have
	\begin{equation*}
	\ProbData \left(
	\xi\,\le\,
	\frac{1}{\delta}\EspData \xi
	\right)\, \geq\, 1-\delta\,,
	\end{equation*}
	which, combined with Eq.~\eqref{eq:newhope}, gives
		\begin{align*}
	  	\ProbData \left(\EspModel_{f\sim\post} \lambda\, \big(\genloss(f)-\emplossStat(f)\big)
		\leq \KL(\post\|\prior) + \ln \bigg(
	 \frac{1}{\delta}\EspData \xi
		\bigg)\right)\\
		\geq 1-\delta\,.
	\end{align*}
	By rearranging the terms of above equation, we obtain the following equivalent form of the statement of the theorem:
		\begin{align*}
        \ProbData\Bigg( &
    \EspModel_{f\sim\post} \genloss(f) 
    	 \le \  \EspModel_{f\sim\post} \emplossStat(f)   \\[-2mm]
    	 &\   + \dfrac{1}{\lambda}\!\left[ \KL(\post\|\prior) +
    	\ln \bigg(
	 \frac{1}{\delta}\EspData \xi
		\bigg)\right] \Bigg) \geq 1-\delta \,.
    \end{align*}
    To see that the inequality above is
    equivalent to the statement of the theorem,
    note that by Fubini's theorem,
    $$\EspData \xi=\EspData\EspModel_{f\sim\post} e^{\lambda \big(\genloss(f)-\emplossStat(f)\big)}=\EspModel_{f\sim\post} \EspData e^{\lambda \big(\genloss(f)-\emplossStat(f)\big)},$$
      and hence
     $\ln \EspData \xi =\Psi_{\ell,\prior}(\lambda,n)$.
     Moreover, 
    $\ln(
	 \frac{1}{\delta}\EspData \xi)=\ln \frac{1}{\delta} + \ln \EspData \xi$.
\end{proof}

\subsection{Details leading to Eq.~\eqref{eq:leqL}}
 For any $f\in\Fcal$:
\begin{align*}
    &\EspData
    	\exp\left[{\lambda\left(\genloss(f) - \emplossStat(f)\right)}\right]\\
    	=&
    \EspData 
    	e^{\frac\lambda{n} \sum_{i=1}^n \left(\EspData \ell(f(\bX_k),\bY_k) - \ell(f(\bX_i),\bY_i)\right)}\\
    	=&
    \EspData 
    	\prod_{i=1}^n e^{\frac\lambda{n} \left(\EspData \ell(f(\bX_k),\bY_k) - \ell(f(\bX_i),\bY_i)\right)}\\
    	\textrm{($\bX_i,\bY_i$ \iid)}=&
    	\prod_{i=1}^n \EspData  e^{\frac\lambda{n} \left(\EspData \ell(f(\bX_k),\bY_k) - \ell(f(\bX_i),\bY_i)\right)}\\
    	\textrm{(Hoeff.)}\leq&
    	\prod_{i=1}^n  \exp\left[{\frac{\lambda^2 L^2}{8 n^2} }\right] \\
    	= &
    	 \exp\left[{\frac{\lambda^2 L^2}{8 n} }\right],
\end{align*}
where the line (Hoeff.) is obtained from Hoeffding's lemma on the random variable 
$\left(\genloss(f) - \ell(f(\bX_i),\bY_i)\right) \in [-\genloss(f), L-\genloss(f) ]$, which has an expected value of zero.

\bibliography{biblio.bib}
\bibliographystyle{plainnat}

\end{document}